\documentclass[twoside]{article}

\usepackage[accepted]{aistats2026}
\usepackage[round]{natbib}

\usepackage{amsmath}
\usepackage{amssymb}
\usepackage{mathtools}
\usepackage{amsthm}
\usepackage{physics}
\usepackage{mathrsfs}
\mathtoolsset{showonlyrefs}
\usepackage{enumerate}
\usepackage{lastpage}
\usepackage{float}
\usepackage{verbatim}
\usepackage{microtype}
\usepackage{graphicx}
\usepackage{subfigure}
\usepackage{booktabs} 
\usepackage{xcolor}

\newcommand{\fS}{\mathcal{S}}
\newcommand{\fA}{\mathcal{A}}
\newcommand{\fY}{\mathcal{Y}}

\newcommand{\fO}{\mathcal{O}}

\newcommand{\fV}{\mathcal{V}}

\newcommand{\R}{\mathbb{R}}

\newcommand{\E}{\mathbb{E}}

\newcommand{\ns}{{|\fS|}}
\newcommand{\na}{{|\fA|}}

\theoremstyle{plain}
\newtheorem{theorem}{Theorem}[section]

\newtheorem{lemma}[theorem]{Lemma}
\newtheorem{corollary}[theorem]{Corollary}
\theoremstyle{definition}
\newtheorem{definition}[theorem]{Definition}
\newtheorem{assumption}[theorem]{Assumption}
\theoremstyle{remark}

\begin{document}

%
\runningtitle{\runningtitle{Almost Sure Convergence of Differential Temporal Difference Learning}}

%
\runningauthor{Blaser, Wang, Zhang}

\twocolumn[

\aistatstitle{Almost Sure Convergence of Differential Temporal Difference Learning for Average Reward Markov Decision Processes}

\aistatsauthor{ Ethan Blaser \And Jiuqi Wang \And  Shangtong Zhang }

\aistatsaddress{University of Virginia \\ blaser@email.virginia.edu \And University of Virginia \\jiuqi@email.virginia.edu \And University of Virginia \\ shangtong@virginia.edu}]

\begin{abstract}
    The average reward is a fundamental performance metric in reinforcement learning (RL) focusing on the long-run performance of an agent. Differential temporal difference (TD) learning algorithms are a major advance for average reward RL as they provide an efficient online method to learn the value functions associated with the average reward in both on-policy and off-policy settings. However, existing convergence guarantees require a local clock in learning rates tied to state visit counts, which practitioners do not use and does not extend beyond tabular settings. We address this limitation by proving the almost sure convergence of on-policy $n$-step differential TD for any $n$ using standard diminishing learning rates without a local clock. We then derive three sufficient conditions under which off-policy $n$-step differential TD also converges without a local clock. 
    These results strengthen the theoretical foundations of differential TD and bring its convergence analysis closer to practical implementations.
\end{abstract}

\section{Introduction}
The average reward is an important performance metric in Reinforcement Learning (RL, \citet{sutton2018reinforcement}).
Compared with the commonly used discounted total rewards performance metric,
the average reward setting more heavily emphasizes the long-term behavior of the RL agent,
making it particularly suitable for applications like network resource allocation~\citep{marbach1998call, bakhshi2021r, yang2024average}, robotics \citep{kober2013reinforcement}, and scheduling \citep{ghavamzadeh2007hierarchical}.

Differential temporal difference (TD) \citep{wan2020learning} learning is one of the most important recent advances for average reward RL.
Differential TD is designed to estimate the corresponding value function for the average reward performance metric and can be used in both on-policy and off-policy settings.
However, the convergence analysis of differential TD remains less satisfactory. In \citet{wan2020learning}, almost sure convergence is proved only when the stepsizes depend on a local clock.
Specifically,
they require the learning rates of the form $\qty{\alpha_{\nu(t, S_t)}}$, where $\qty{\alpha_t}$ is a sequence of deterministic, nonnegative, and diminishing scalars and a local clock (i.e., a counter) $\nu(t, s)$,
which counts the number of visits to a state $s$ up to timestep $t$. 
In other words, at time $t$ the stepsize depends not only on $t$, but also on the number of past visits to the current state $S_t$.

We argue that this local clock based learning rate is unsatisfactory for at least three reasons.
First,
to our best knowledge,
practitioners do not actually use the local clock in their learning rates, including \citet{wan2020learning} in their experiments.
The local clock seems to be primarily a theoretically motivated technique \citep{borkar2009stochastic}. Although recent work demonstrates that it can occasionally be required for convergence in certain settings \citep{chen2025non}, its adoption in practical implementations remains rare.
Second,
the local clock cannot be used in many function approximation settings,
especially those considered in \citet{sutton2018reinforcement},
where the agent only has access to the feature of the current state, denoted as $\phi(S_t)$, not the state $S_t$ itself.
With only $\phi(S_t)$,
it is not clear how to count the visits to $S_t$ since the feature function $\phi$ is usually not a one-to-one mapping.
This means the local clock technique is only viable in the tabular setting.
Third,
although convergence analyses of discounted TD \citep{sutton1988learning} also require the local clock in learning rates \citep{jaakkola1993convergence,tsitsiklis1994asynchronous},
later works removed this requirement \citep{tsitsiklis1997analysis,liu2024ode}.
Therefore, there is a theoretical gap in the literature for average reward RL, and gives rise to the central question this work aims to answer:
\begin{center}
    \emph{Can we establish the convergence of differential TD without using a local clock in the learning rates?}
\end{center}
This question seems trivial at first glance.
After all, local clocks can be avoided in the discounted setting, so one might expect the same argument to carry over to the average reward setting. However, as we will now explain, extending that analysis to the average-reward case introduces several fundamental obstacles.

The convergence of the discounted TD with a local clock rests on the global asymptotic stability (G.A.S.) of the following ODE\footnote{The ODE~\eqref{eq ode 1} is G.A.S. if and only if the $A$ matrix is Hurwitz (Theorem 4.5 from \citet{khalil2002nonlinear}). A matrix $A$ is Hurwitz if the real part of any of its eigenvalues is strictly negative.}
\begin{align}
    \dv{v(t)}{t} = Av(t) \label{eq ode 1},
\end{align}
where $v(t) \in \R^\ns$ can be viewed as the estimation of the value function and $A \in \R^{\ns \times \ns}$ corresponds to the discounted TD algorithm,
with $\ns$ being the number of states.
Essentially,~\eqref{eq ode 1} is G.A.S. because the $A$ matrix corresponding to discounted TD with a local clock is negative definite (n.d.)\footnote{A matrix $A$, not necessarily symmetric, is n.d. if for any $y \neq 0$, it holds that $y^\top A y < 0$. A n.d. matrix must be Hurwitz. But a Hurwitz matrix does not need to be n.d. For example, $\mqty[-1 & 10 \\0 & -1]$ is Hurwitz but not n.d.}.
When the local clock is removed from the learning rates,
the corresponding ODE becomes
\begin{align}
    \dv{v(t)}{t} = DAv(t) \label{eq ode 2},
\end{align}
where $D \in \R^{\ns \times \ns}$ is a diagonal matrix whose entries are the stationary state distribution. 
The change from~\eqref{eq ode 1} to~\eqref{eq ode 2} is intuitive. With a local clock, the total magnitude of updates applied to each state is forced to be the same, regardless of how frequently that state is visited. Without a local clock, the magnitude of the updates naturally depends on visitation frequency, which appears as the multiplier $D$ in~\eqref{eq ode 2}. For instance, when a state $s$ is visited for the first time, the learning rate is always $\alpha_1$ with a local clock, whereas without it the learning rate may be $\alpha_{100}$ if $s$ is first visited at time $t=100$.
Nevertheless, when $A$ is n.d., it is straightforward to show that $DA$ is also n.d., implying that~\eqref{eq ode 2} is G.A.S., and thus that discounted TD converges even without a local clock.

However,
as we shall show soon,
the corresponding $A$ matrix for differential TD with the local clock is only Hurwitz and not necessarily n.d.
When $A$ is Hurwitz,
whether $DA$ is also Hurwitz is a long-standing open problem in the linear algebra community,
called the $D$-stability problem \citep{johnson1974sufficient,giorgi2015overview}.
Progress on the $D$-stability problem has been limited in the past decade \citep{kushel2019unifying,kushel2023novel, tong2024sufficient}.
As a result,
verifying whether~\eqref{eq ode 2} is G.A.S. for differential TD without a local clock is substantially more challenging than it initially appears.

Nevertheless,
this paper makes three contributions.
First,
we establish the almost sure convergence of on-policy $n$-step differential TD for any $n$ without the local clock.
Second,
we give three different sufficient conditions for the almost sure convergence of off-policy $n$-step differential TD without the local clock.
Admittedly,
our characterization in the off-policy case is incomplete and we correspondingly present our third contribution:
we outline a few challenges and open problems in this area.

\section{Background}\label{sec: background}

In this work, all vectors are column. The $\ell_2$ norm in $\R^d$ is denoted by $\norm{\cdot}$. The identity matrix is denoted by $I$, and we use $e$ to denote the all-one vector. For a matrix $A \in \R^{n\times n}$, we denote its spectral radius by $\lambda_{\max}(A) \doteq \max \qty{\abs{\lambda} : \lambda \in \sigma(A)}$, with $\sigma(A)$ as the set of eigenvalues of $A$. We say that a matrix $A$ is (strictly) \textit{positive stable} if $\forall \lambda \in \sigma(A)$ $\Re \lambda \geq 0$ ($\Re \lambda > 0$). 
It is easy to see that $A$ is strictly positive stable if and only if $-A$ is Hurwitz.\footnote{``Hurwitz'' is often used in the control community while ``positive stable'' is often used in the linear algebra community}
If a matrix $A$ has only nonnegative (positive) entries, we write $A \geq 0$, ($A>0$). If a matrix $A$ is positive definite, we write $A \succ 0$. Given any vector $x$, $\sum x$ denotes the sum of all elements in $x$. We use $A_{i,j}$ to refer to the $(i, j)$-th entry of $A$.
\begin{definition} \label{def: m matrix}
An $M-$matrix is a matrix of the form $\gamma I - P$ where $P \in \R^{n\times n}$, $P \geq 0$, and $\gamma \geq \lambda_{\max}(P)$. 
\end{definition}

In RL, we consider a Markov Decision Process (MDP; \citet{bellman1957markovian,puterman2014markov}) with a finite state space $\fS$,
a finite action space $\fA$,
a reward function $r: \fS \times \fA \to \R$,
a transition function $p: \fS \times \fS \times \fA \to [0, 1]$,
an initial distribution $p_0: \fS \to [0, 1]$.
At time step $0$,
an initial state $S_0$ is sampled from $p_0$.
At time $t$,
given the state $S_t$,
the agent samples an action $A_t \sim \pi(\cdot | S_t)$, 
where $\pi: \fA \times \fS \to [0, 1]$ is the policy being followed by the agent.
A reward $R_{t+1} \doteq r(S_t, A_t)$ is then emitted and the agent proceeds to a successor state $S_{t+1} \sim p(\cdot | S_t, A_t)$. 

We assume the Markov chain $\qty{S_t}$ induced by the policy $\pi$ is ergodic and thus adopts a unique stationary distribution $d_\pi$. We define $D_\pi = \text{diag}(d_\pi)$.
The average reward (a.k.a. gain, \citet{puterman2014markov}) is defined as $\bar{J}_{\pi} \doteq \lim_{T\rightarrow \infty} \frac{1}{T}\sum_{t=1}^T \E\left[R_t\right].$
Consequently,
the differential value function (a.k.a. bias, \citet{puterman2014markov}) is defined as $v_\pi(s) \doteq \lim_{T\to\infty} \frac{1}{T} \sum_{\tau=1}^T \E\left[ \sum_{i=1}^{\tau} (R_{t+i} - \bar{J}_{\pi}) \mid S_t = s\right].$
The corresponding Bellman equation (a.k.a. Poisson's equation) is then
\begin{align}
    \label{eq bellman equation}
    v = r_\pi -\bar{J}_{\pi}e + P_\pi v ,
\end{align} 
where $v \in \R^\ns$ is the free variable,
$r_\pi \in \R^\ns$ is the reward vector induced by the policy $\pi$, i.e., $r_\pi(s) \doteq \sum_a \pi(a|s) r(s, a)$, and $P_\pi \in \R^{\ns \times \ns}$ is the transition matrix induced by the policy $\pi$, i.e., $P_\pi(s, s') \doteq \sum_a \pi(a|s)p(s'|s, a)$. 
It is known \citep{puterman2014markov} that all solutions to~\eqref{eq bellman equation} form a set $\fV_* \doteq \qty{v_\pi + ce \mid c \in \R}.$
The policy evaluation problem in average reward MDPs is to estimate $v_\pi$,
perhaps up to a constant offset $ce$.

In the off-policy setting, an agent aims to evaluate a target policy $\pi$ but follows a behavior policy $\mu$. We define the importance sampling ratio $\rho(s,a) \doteq \frac{\pi(a|s)}{\mu(a|s)}$ and $\rho_t \doteq \rho(S_t,A_t)$.

\section{Differential Temporal Difference Learning}
\label{sec:avg reward td}  
Differential TD is designed to estimate $v_\pi$ in an online manner.
The differential TD algorithm proposed by \citet{wan2020learning} only considers a one-step look-ahead.
Inspired by the success of $n$-step TD in the discounted setting \citep{sutton2018reinforcement},
we first extend the 1-step differential TD to the $n$-step case.
As we shall see soon,
this extension is vital to the analysis in the off-policy setting.
Here we only present off-policy $n$-step differential TD as the on-policy version is just a special case with $\mu = \pi$.
Suppose a trajectory $\{S_0, A_0, R_1, S_1, \dots\}$ is generated by following a behavior policy $\mu$ as $A_t \sim \mu(\cdot \mid S_t)$.
Since the $n$-step return for $S_t$ is only observable after reaching $S_{t+n}$,
the iterates $\qty{v_t \in \R^\ns}, \qty{J_t \in \R }$ are updated at time $t+n$ as
\begin{align}
    &\delta_t = \textstyle R_{t+1:t+n} - nJ_{t+n-1} + v_{t+n-1}(S_{t+n}) - v_{t+n-1}(S_t), \\
    &J_{t+n} = \textstyle J_{t+n-1} + \frac{\eta}{n}\alpha_{t+n-1} \rho_{t:t+n-1}\delta_t, \\
    &v_{t+n}(S_t) = v_{t+n-1}(S_t) + \alpha_{t+n-1}\rho_{t:t+n-1}\delta_t,     \label{eq diff td full}
\end{align}  
where $\rho_{t:t+n-1} \doteq \prod_{k=t}^{t+n-1} \rho_k$ and $R_{t+1:t+n} \doteq \sum_{k=1}^n R_{t+k}$ are shorthands.

To our knowledge,
this is the first time that $n$-step differential TD is formalized, and the complete derivation is presented in Appendix \ref{ap: derivation}.
When $n=1$,
it recovers the $1$-step differential TD in \citet{wan2020learning}.
However,
in the convergence analysis in \citet{wan2020learning},
they replace the learning rate $\qty{\alpha_t}$ with $\qty{\alpha_{\nu(t, S_t)}}$.
We recall that $\nu(t, s)$ counts the number of visits to the state $s$ until time $t$ and is referred to as the local clock.
In this work,
we shall conduct our analysis of~\eqref{eq diff td full} directly without altering the learning rates.

Inspired by \citet{wan2020learning}, to facilitate our analysis,
we first rewrite~\eqref{eq diff td full} to eliminate the iterates $\{J_t\}$.
Define $\Sigma_t \doteq \sum_s v_t(s)$. Since the $n-$step return for $S_t$ is only available after time $t+n$, we adopt the standard convention that no updates occur before the first $n-$step return is observed, so $J_t$ and $v_t$ are constant for $t<n$. Making use of the fact that $v_{t+n}$ and $v_{t+n-1}$ differ from each other only for the $S_t$-indexed entry,
we obtain
\begin{align}
    J_{t+n} - J_{n-1} 
    &= \textstyle \sum_{i=0}^{t} \frac{\eta}{n}\alpha_{i+n-1}\rho_{i:i+n-1}\delta_i, \\
    &= \textstyle \sum_{i=0}^{t} \frac{\eta}{n}\qty(v_{i+n}(S_i) - v_{i+n-1}(S_i)), \\
    &= \textstyle \sum_{i=0}^{t} \frac{\eta}{n}\sum_s \qty(v_{i+n}(s) - v_{i+n-1}(s)), \\
    &= \textstyle \frac{\eta}{n}\,(\Sigma_{t+n} - \Sigma_{n-1}).
\end{align}
We can then rewrite $\delta_t$ from~\eqref{eq diff td full} as
\begin{align}
    \delta_t &= R_{t+1:t+n} - n\qty(J_{n-1} + \tfrac{\eta}{n}(\Sigma_{t+n-1} - \Sigma_{n-1})) 
               \\ &\quad + v_{t+n-1}(S_{t+n}) - v_{t+n-1}(S_t).
\end{align}
Then,~\eqref{eq diff td full} can be expressed more compactly as
\begin{align}
    &v_{t+n}(S_t) = v_{t+n-1}(S_t) + \alpha_{t+n-1}\rho_{t:t+n-1}\bigl(\textstyle
        \tilde R_{t+1:t+n} \\&\quad  \textstyle - \eta \Sigma_{t+n-1}
        + v_{t+n-1}(S_{t+n}) - v_{t+n-1}(S_t)\bigr),\label{eq: diff TD}
\end{align}
where $\tilde R_{t+1:t+n} \doteq \sum_{k=1}^n (R_{t+k} - J_{n-1} + \tfrac{\eta}{n}\Sigma_{n-1})$.
We assume initialization with $J_{n-1} \doteq 0$ and $v_{n-1} \doteq 0$ for simplifying presentation so that $\tilde R_{t+1:t+n} = R_{t+1:t+n}$.
For nonzero initialization of $J_0$ and $v_0$,
we only need to conduct the same analysis in a new MDP with a shifted reward function $\tilde{r}(s, a) \mapsto r(s, a) - J_{n-1} + \frac{\eta}{n} \Sigma_{n-1}$ \citep{wan2020learning}.

\section{Convergence of Differential Temporal Difference Learning}

We make the following standard assumptions.

\begin{assumption}[Ergodicity and coverage]\label{as: markov}
The Markov chains induced by the behavior policy $\mu$ and target policy $\pi$ are finite, irreducible, and aperiodic. The behavior policy $\mu$ covers $\pi$ i.e. $\forall\,s\in\mathcal{S},\ \forall\,a\in\mathcal{A}:\quad \pi(a\mid s)>0\implies\mu(a\mid s)>0.$
\end{assumption}
The ergodicity assumption is standard for analyzing RL algorithms \citep{bertsekas1996neuro}. Furthermore, the coverage assumption is the same as in \citet{sutton2018reinforcement}. 

From Assumption \ref{as: markov}, the Markov chains induced by the behavior policy and target policy each adopt a unique stationary distribution, which we denote respectively as $d_\pi$ and $d_\mu$. Because the Markov chains are irreducible, $d_\pi > 0$ and $d_\mu >0$ \citep{puterman2014markov}.

\begin{assumption}\label{as: lr}
    The learning rates $\qty{\alpha_t}$ are positive, decreasing, and satisfy
    \begin{align}
        \textstyle \sum_{t=0}^\infty \alpha_t = \infty, \lim_{t \rightarrow \infty} \alpha_t = 0, \, \text{and} \, \frac{\alpha_t - \alpha_{t+1}}{\alpha_t} = \fO(\alpha_t).
    \end{align}
\end{assumption}
This is the standard set of assumptions for learning rates in stochastic approximation \citep{borkar2009stochastic}. We emphasize that this definition of the learning rates is far more widely used compared to the state visitation-dependent learning rates found in \citet{wan2020learning}. For example, Assumption \ref{as: lr} is satisfied by any learning rate of the form $\alpha_t = \frac{C_1}{(n + C_2)^\beta}$
where $C_1$ and $C_2$ are constants and $\beta \in (0.5,1]$.

Our proof of convergence will utilize some results from the stochastic approximation (SA) community. Thus, we begin by writing the update \eqref{eq: diff TD} as a canonical stochastic approximation update by first defining an augmented Markov chain $\qty{Y_t}$ evolving in a finite state space $\fY$ as,
\begin{align} \label{eq: markov auxiliary}
    Y_{t+1} = \qty(S_t, A_t, S_{t+1}, A_{t+1}, \dots, A_{t+n-1}, S_{t+n}).
\end{align}
From Assumption \ref{as: markov}, it is clear that $\qty{Y_{t}}$ is also irreducible and aperiodic, and we denote its stationary distribution as $d_\fY$.
Then, we can define the operator $H: \R^{\abs{\fS}} \times \fY \to \R^{\abs{\fS}}$ as,
\begin{align}
    H\qty(v,y)[s] &\doteq \textstyle  \rho_{0:n-1}(y)\bigl(\sum_{k=0}^{n-1}r(s_k,a_k)
    -\eta \sum v  \\ 
    &\quad +v(s_n)-v(s_0)\bigr)\mathbb{I}\qty{s=s_0}. \label{eq: H def}
\end{align}
where we have $y \doteq (s_0, a_0, \dots s_n)$ and $\rho_{0:n-1}(y) \doteq \prod_{k=0}^{n-1} \rho(s_k, a_k)$.
Then we can write \eqref{eq: diff TD} as a canonical SA algorithm according to
\begin{align}
    v_{t+1} = v_t + \alpha_t H(v_t, Y_{t+1}). \label{eq: SA}
\end{align}
Note that the SA iteration index $t$ differs from the environment time step in \eqref{eq: diff TD}. 
One SA update corresponds to an $n$-step block of experience, so SA step $t$ corresponds to environment time $t+n-1$, since the tuple $Y_{t+1}$ becomes available only after observing up to $S_{t+n}$.

One prominent method for analyzing the asymptotic behavior of $\qty{v_t}$ is to regard $\qty{v_t}$ as Euler's discretization of the ODE
\begin{align}
\dv{v(t)}{t} &= h(v(t)) \label{eq: ode method}
\end{align}
where the expected operator $h(v) \doteq \E_{y \sim d_{\fY}}\qty[H(v,y)]$. Using this method, the asymptotic behavior of the discrete and stochastic updates $\qty{v_t}$ can be characterized by the continuous and deterministic trajectories of the ODE \eqref{eq: ode method}, if the stability of the iterates can be established. The Borkar-Meyn theorem \citep{borkar2000ode} establishes the desired stability given the ODE@$\infty$ is G.A.S., which is defined as
\begin{align}
\dv{v(t)}{t} &= h_\infty(v(t)),
\end{align}
where $h_\infty \doteq \lim_{c\rightarrow \infty} \frac{h(cv)}{c}$. Although the original work of \citet{borkar2000ode} only allows for $\qty{Y_t}$ to be i.i.d, recently \citet{liu2024ode} generalized the Borkar-Meyn theorem to Markovian noise $\qty{Y_t}$ under equally mild assumptions, an important extension which we will leverage here since our $\qty{Y_t}$ in \eqref{eq: markov auxiliary} is a Markov chain.

We therefore proceed by studying the expected operator for \eqref{eq: diff TD}:
\begin{align}
    h(v)[s] &=\E_{y\sim d_{\fY}}\qty[H(v,y)[s]] \\
    &= \textstyle \E_{\mu} \Bigl[
          \rho_{0:n-1}(y)\bigl(\sum_{k=0}^{n-1} r\qty(s_k, a_k)
           -\eta \sum v \\
    &\quad \quad +v\qty(s_n) -v(s)\bigr)\mathbb{I}\qty{s = s_0}\Bigr]\\
    &= \textstyle d_{\mu}(s) \E_\mu \Bigr[\rho_{0:n-1}(y) \Bigl(\sum_{k=0}^{n-1} r\qty(s_k, a_k)\\
    &\textstyle \quad \quad -\eta \sum v + v\qty(s_n) -v(s) \Bigr) \Big| s_0 =s \Bigl],
    \end{align}
where we have used $\E_{\mu}$ to abbreviate the expectation over the trajectory generated by $a_k \sim \mu(\cdot|s_k)$ and $s _{k+1} \sim p(\cdot| s_{k}, a_k)$.
Isolating the reward-sum term, we define the expected $n$-step reward by
\begin{align}
r^{(n)}(s)
&\doteq \textstyle \E_{\mu} \qty[
\rho_{0:n-1}\sum_{k=0}^{n-1}r\qty(s_k, a_k)
\Big|s=s_0 ] \\
&= \textstyle \E_{\pi} \qty[\sum_{k=0}^{n-1}r\qty(s_k, a_k) \Big|s=s_0 ] \\
&= \textstyle \sum_{k=0}^{n-1} \qty(P_{\pi}^k r_{\pi})(s).
\end{align}
Therefore, the expected operator for \eqref{eq: diff TD} is
    \begin{align}
        h(v) &= D_{\mu} (r^{\qty(n)} - \eta e e^\top v + P_{\pi}^n v - v), \label{eq: h}
    \end{align}
where $D_{\mu}$ is a diagonal matrix whose entries are $d_\mu$. 
The corresponding ODE@$\infty$ is
\begin{align}
\dv{v(t)}{t} &= h_\infty(v(t)) = D_\mu\qty(P_\pi^n - I - \eta ee^\top) v(t) = -A v(t), \label{eq: A definition}
\end{align}
where $A \doteq D_\mu\qty(I - P_\pi^n + \eta e e^\top)$. 

We now outline the structure of our proof.
The first milestone is to prove that 
the ODE~\eqref{eq: A definition} is G.A.S.
It is then trivial to see that~\eqref{eq: ode method} is also G.A.S..
We will use $v_\infty$ to denote the G.A.S. equilibrium of~\eqref{eq: ode method} and we have $h(v_\infty) = 0$.
This means that $r^{\qty(n)} - \eta e e^\top v_\infty + P_{\pi}^n v_\infty - v_\infty = 0$.
The analysis in Appendix B.2.1 of \citet{wan2020learning}, which we omit to avoid redundancy, then immediately confirms that $v_\infty \in \fV_*$.
The second milestone is to invoke a result from \citet{liu2024ode} (stated as Lemma \ref{lem: liu ode convergence} in the Appendix) to prove that the iterates $\qty{v_t}$ generated in \eqref{eq: diff TD} converge to $v_\infty$ almost surely. We now proceed to carry out this proof strategy in detail.

It is known that a necessary and sufficient condition for \eqref{eq: A definition} to be G.A.S. is that $A$ is strictly positive stable (Theorem 4.5 from \citet{khalil2002nonlinear}). 
\citet{wan2020learning} essentially prove that the matrix $I - P_\pi^n + \eta ee^\top$ is strictly positive stable.
However,
this does not mean that the $A$ matrix is strictly positive stable.
This is an instance of the $D$-stability problem.
As discussed in Section~\ref{sec: related},
this is a very challenging problem in the linear algebra community. 
Nevertheless,
to prove $A$ is strictly positive stable, we will utilize the results from \citet{bierkens2014singular}, which we present as Lemma \ref{thm: bierkens}, that establish conditions under which $M$-matrices (see Definition \ref{def: m matrix}) are strictly positive stable under rank one perturbations.


\newcounter{c}

\begin{lemma}{(Theorem 2.7 from \citet{bierkens2014singular})}. \label{thm: bierkens}
    Let $B \in \R^{n\times n}$ and $v, w \in \R^{n}$. Then $B + vw^\top$ is strictly positive stable if:
    \begin{enumerate}
        \item $B = \lambda_{\max}(K)I - K $ is a singular $M$-matrix where $K \in \R^{n \times n}$.
        \item $0$ is a geometrically simple eigenvalue of $B$ with left and right eigenvectors $z_l \neq 0$ and $z_r \neq 0$. (i.e. $z_l^\top B = 0$ and $B z_r =0$).
        \item $(z_l ^\top v)(w^\top z_r) \neq 0$
        \setcounter{c}{\value{enumi}}
    \end{enumerate}
    and \textbf{either} of the following conditions hold:
    \begin{enumerate}
        \setcounter{enumi}{\value{c}}
        \item $Bv = 0$, or $w^\top B = 0$. 
        \item $v,w > 0$ and $2K_{i,j} \geq v_i w_j \, \forall \, i,j$ (where $v_i$, respectively, $w_j$ denote the $i$-th entry of $v$ and the $j$-th entry of $w$)
    \end{enumerate}
\end{lemma}

To utilize Lemma \ref{thm: bierkens} to prove the strict positive stability of $A$ from \eqref{eq: A definition}, we begin by decomposing $A$ into the form of $B+vw^\top$ with,
\begin{align}
A &= D_\mu\bigl(I - P_\pi^n + \eta\,e\,e^\top\bigr)\\
  &= I - I + D_\mu(I - P_\pi^n) + \eta\,d_\mu\,e^\top\\
  &= I - \qty(I + D_\mu(P_\pi^n - I)) + \eta d_\mu e^\top \\
  &\doteq B + \eta d_\mu e^\top, \label{eq: B definition}
\end{align}
where $B \doteq I - \qty(I + D_\mu(P_\pi^n - I))$, and we recall that $\eta$ is a positive constant.

Without any additional assumptions, we can verify the first three conditions of Lemma \ref{thm: bierkens} in the following Lemma.
\begin{lemma}\label{lem: first three conds}
    Let Assumption \ref{as: markov} hold. Then, $B \doteq I - \qty(I + D_\mu(P_\pi^n - I))$, $v \doteq \eta d_\mu$, and $w \doteq e$ satisfy conditions 1-3 of Lemma \ref{thm: bierkens}.
\end{lemma}
\begin{proof}
    First, we verify Condition 1. If we define $K \doteq I + D_\mu(P_\pi^n - I)$, we have $B = I - K$. Therefore, it is sufficient to prove that $\lambda_{\max}\qty(K) = 1$.
Since $D_\mu=\text{diag}(d_\mu)$ with $d_\mu > 0$ and $P_\pi^n$ is non‐negative, it is easy to see that 
\begin{align}
K \doteq I + D_\mu(P_\pi^n - I) = (I - D_\mu) + D_\mu P_\pi^n \label{eq: K def}
\end{align}
is non-negative. Additionally, computing the row sums of $K$, we see that it is row-stochastic:
\begin{align}
K e &= (I-D_\mu)e+ D_\mu P_\pi^n e =e - D_\mu e + D_\mu e = e, \label{eq: K stochastic}
\end{align}
where the second equality holds because the transition matrix $P_\pi^n$ is row-stochastic.
Then, we are guaranteed that $\lambda_{\max}(K) = 1$ (Theorem 8.1.22 from \citet{horn2012matrix}).

To verify Condition 2, we demonstrate that $B$ has $0$ as an algebraically (and therefore geometrically) simple eigenvalue with left and right eigenvectors $z_l, z_r \neq 0$ (i.e. $z_l ^\top B = 0$ and $B z_r = 0$). We have
\begin{align}
\ker B &= \ker\qty(I - K)
= \qty{z_r: K z_r = z_r}. \label{eq: right kernel}
\end{align}
Additionally, since
$B^\top=(I-K)^\top$
we have
\begin{align}
\ker B^\top
=\ker\qty(I-K)^\top
=\qty{z_l : z_l^\top K = z_l^\top}.
\end{align}
Since $P_\pi$ is irreducible and aperiodic by Assumption \ref{as: markov}, $P_\pi^n$ is irreducible for every $n\geq1$ \citep{puterman2014markov}. Multiplying a positive diagonal matrix and adding another positive diagonal matrix leaves the zero pattern unchanged so $K = I - D_\mu + D_\mu P_\pi^n$ is also irreducible. Therefore, the Perron-Frobenius theorem (Theorem 8.4.4 in \citet{horn2012matrix}) guarantees that one is an algebraically simple eigenvalue of $K$. This implies that $\ker B$ and $\ker B^{\top}$ are one‑dimensional, and thus zero is a geometrically simple eigenvalue of $B$. We identify the one-dimensional left and right kernels of $B$ as,
\begin{align}
    \ker{B} = \text{span}(e), \quad \ker{B^\top} = \text{span}\qty(d_\pi / d_\mu),
\end{align}
where $d_\pi / d_\mu$ represents element-wise division. For the right kernel, it holds trivially from \eqref{eq: right kernel} and the fact that $K$ is row-stochastic. For the left kernel, with $B^\top = I- K^\top = D_\mu - P_\pi^{n\top} D_\mu$ we have,
\begin{align}
    \textstyle B^\top \qty(\frac{d_\pi}{d_\mu}) = D_\mu \qty(\frac{d_\pi}{d_\mu}) - P_\pi^{n\top} D_\mu \qty(\frac{d_\pi}{d_\mu}) = d_\pi - P_\pi^{n\top} d_\pi = 0.
\end{align}
Clearly $z_l = d_\pi / d_\mu$ and $z_r = e$ are both nonzero, so Condition 2 is satisfied. 

To verify Condition 3, we note that all components of $e$ and $\frac{d_\pi}{d_\mu}$ are strictly positive, so any non-zero vector $z_l^\top \in \text{span}(\frac{d_\pi}{d_\mu})$ and $z_r \in \text{span}(e)$ will have uniform sign. Using the fact that $v = \eta d_\mu$ and $w=e$ are strictly positive, it is easy to see that
\begin{equation}
    (z_l^\top v)(w^\top z_r) = \eta(z_l^\top d_\mu)(e^\top z_r) \neq 0.
\end{equation}
\end{proof}

Although we have verified Conditions 1-3 of Lemma \ref{thm: bierkens} using only Assumption \ref{as: markov}, we still need either Condition 4 or 5 to establish that $A$ is strictly positive stable. We therefore split the analysis into two regimes. In the on-policy case with $\mu=\pi$ (Section \ref{sec: on policy}), we are able to directly satisfy Condition 4. In the off-policy case (Section \ref{sec: off policy}), 
additional restrictions are needed, and we provide three sufficient conditions.

\subsection{On-Policy Case} \label{sec: on policy}
In the on-policy case, we consider the following assumption.
\begin{assumption}[On-policy] \label{as: on policy} The behavior policy followed by the agent is the target policy, i.e., $\pi(a|s) = \mu(a|s) \, \forall s \in \fS, a\in \fA$. 
\end{assumption} 

To prove the strict positive stability of $A$ in the on-policy case, since Conditions 1-3 are already in place from Lemma \ref{lem: first three conds}, it remains only to verify Condition 4. Theorem \ref{thm: on policy stability} does so, thereby establishing the strict positive stability of $A$. Corollary \ref{corr: on policy convergence} then gives the almost-sure convergence of Differential TD.

\begin{theorem} \label{thm: on policy stability}
    Let Assumptions \ref{as: markov} and \ref{as: on policy} hold. Then, $A = B + \eta d_\mu e^\top$ is strictly positive stable for any $n\geq 1$ and any $\eta > 0$.
\end{theorem}
\begin{proof}
Recall from \eqref{eq: B definition}, we have expressed $A$ in the form of $B+vw^\top$ where $B \doteq I - \qty(I + D_\mu(P_\pi^n - I))$, $v= \eta d_\mu$ and $w = e$. Lemma \ref{thm: bierkens} states that $A$ is strictly positive-stable once Conditions 1-3 together with either Condition 4 or 5, are satisfied.
In Lemma \ref{lem: first three conds} we verify the first three conditions of Lemma \ref{thm: bierkens} with this choice of $B, v, w$. In the on-policy case (Assumption \ref{as: on policy}), we have $\mu = \pi$, which further gives
\begin{align}
    B = I - \qty(I + D_\pi(P_\pi^n - I)), \quad v = \eta d_\pi, \quad w=e.
\end{align}

To demonstrate Condition 4 holds in the on-policy setting, we show $w^\top B = 0$ with,
\begin{align}
w^\top B = e^\top D_\pi(I-P_\pi^n)
  = d_\pi^\top(I-P_\pi^n)
  = 0. 
\end{align} 
\end{proof}

\begin{corollary}\label{corr: on policy convergence}
    Let Assumptions \ref{as: markov}, \ref{as: lr} and \ref{as: on policy} hold. Then the iterates $\qty{v_t}$ in \eqref{eq: diff TD} satisfy: $\lim_{t \rightarrow \infty} v_t = v_\infty ,\ a.s.    $, where $v_\infty \in \fV_*$.
\end{corollary}
\begin{proof}
To prove the almost sure convergence of the differential TD iterates in \eqref{eq: diff TD} to fixed point $v_\infty$ we will utilize Corollary 8 from \citep{liu2024ode} which we present as Lemma \ref{lem: liu ode convergence}. We proceed by verifying the requisite Assumptions \ref{as: liu markov}-\ref{as: liu lil}. Starting with Assumption \ref{as: liu ode conv}, in Theorem \ref{thm: on policy stability} we prove that $A$ defined in \eqref{eq: A definition} is strictly positive stable under Assumptions \ref{as: markov}, \ref{as: lr}, \ref{as: on policy}. Therefore, the ODE@$\infty$ in \eqref{eq: A definition}, is G.A.S. (Theorem 4.5 from \citet{khalil2002nonlinear}). 

Verifying the remaining assumptions is straightforward. Note that our Assumptions \ref{as: markov} and \ref{as: lr} are sufficient to directly satisfy Assumptions \ref{as: liu markov}, \ref{as: liu lr} and \ref{as: liu lil}. We refer the reader to Remarks 1-3 of \citep{liu2024ode} for a discussion on how these are trivially satisfied for ergodic and finite $\qty{Y_t}$. We then verify Assumption \ref{as: liu Hc} in Lemma \ref{lem:Hc}. It is easy to verify that $H(x,y)$ is Lipschitz, which we present for completeness in Lemma \ref{lem:H-lipschitz} that satisfies \ref{as: liu lipschitz}.  Then, Lemma \ref{lem: liu ode convergence} guarantees that $\qty{v_t}$ converges to the invariant set of the ODE, which is a singleton we denote as $v_\infty$. 
\end{proof}


\subsection{Off-Policy Case} \label{sec: off policy}

In the off‐policy setting, we consider three additional assumptions. 

We will first prove that there exists some $\eta_0$ for which for $\eta \in (0, \eta_0]$, $A$ is strictly positively stable using an extension of Lemma \ref{thm: bierkens}, Lemma \ref{lem: 2.11 bierkens}.
\begin{lemma}[Lemma 2.11 from \citet{bierkens2014singular}] \label{lem: 2.11 bierkens}
    Let $B \in \R^{n\times n}, v, w \in \R^n, v, w \geq 0$ satisfy Conditions 1-3 of Lemma \ref{thm: bierkens}. Additionally, let $0$ be an algebraically simple eigenvalue of $B$. Define a matrix-valued curve $\Gamma(t): t \rightarrow B + tvw^\top, t \in \R$. There exists a $t_0 > 0$ such that $\Gamma(t)$ is strictly positive stable for $t\in (0, t_0]$.
\end{lemma}

\begin{lemma}\label{lem: exists eta0 bound}
    Let Assumption \ref{as: markov} hold. Then, there exits a $\eta_0 >0$ such that $A = B + \eta d_\mu e^\top$ is strictly positive stable for $\eta \in (0, \eta_0]$. 
\end{lemma}
\begin{proof}
    Recall from \eqref{eq: B definition} that $A=B+\eta d_\mu e^\top$ where $B \doteq I - \qty(I + D_\mu(P_\pi^n - I))$. By Lemma $\ref{lem: first three conds}$, $B$, $d_\mu$, and $e$ satisfy Conditions 1-3 of Lemma \ref{thm: bierkens}. \footnote{In Lemma $\ref{lem: first three conds}$, we prove this for $v= \eta d_\mu$ instead of $v= d_\mu$. However, since $\eta$ is a positive constant, its easy to see that the argument still holds.} Therefore, if we set $v=d_\mu$, $w=e$ and $t=\eta$, then $A= B+tvw^\top$ and Lemma \ref{lem: 2.11 bierkens} guarantees the existence of some $t_0 > 0$ (hence $\eta_0$) for which $\Gamma(t) = B+tvw^\top$ is strictly positive stable on $(0,t_0]$. It remains only to check that $0$ is algebraically simple for $B$. 

To show this, we need to show that $0$ is a simple root of the characteristic polynomial of $B$. We use $\chi_M(\lambda) = \det(M-\lambda I)$ to denote the characteristic polynomial of a matrix $M$. Recall the definition of $K$ from \eqref{eq: K def}. Then, the characteristic polynomial of $B = I-K$ is,
\begin{align}
\chi_{B}(\lambda)&=\det \qty((I-K)-\lambda I) \\
          &=(-1)^{\ns}\det \qty(K-(1-\lambda)I)\\ 
          &=(-1)^{\ns}\chi_{K}(1-\lambda).
\end{align}
We proved in Lemma \ref{thm: on policy stability} that $1$ is an algebraically simple eigenvalue of $K$. In other words, $\chi_K(\kappa)$ has a simple root at $\kappa = 1$. Then the change of variable $\kappa \mapsto 1-\lambda$ implies $\lambda=0$ is a simple root of $\chi_{B}$. Thus $0$ is an algebraically simple eigenvalue of $B$.

 Then, Lemma \ref{lem: 2.11 bierkens} proves that there exists some $\eta_0 > 0$ for which $A = B+ \eta d_\mu e^\top$ is strictly positive stable on $(0,\eta_0]$.
\end{proof}
Having established that $A$ is strictly positive–stable, the almost‐sure convergence of \eqref{eq: diff TD} to $v_\infty$ follows immediately by the same argument used in Corollary \ref{corr: on policy convergence}.  In that proof, every assumption except \ref{as: liu ode conv} was checked without invoking Assumption \ref{as: on policy}, and \ref{as: liu ode conv} itself is a direct consequence of the strict positive stability of $A$. Therefore, we omit the proof of the corollary to avoid redundancy.

\begin{corollary}
    Let Assumptions \ref{as: markov}, \ref{as: lr} hold. Then there exists some positive constant $\eta_0>0$ such that for $\eta \in (0,\eta_0]$ the iterates $\qty{v_t}$ in \eqref{eq: diff TD} satisfy $\lim_{t \rightarrow \infty} v_t = v_\infty ,\ a.s.$, where $v_\infty \in \fV_*$.
\end{corollary}

The main limitation of this result is that, while it guarantees some $\eta_0>0$, it does not offer a closed form for its value. To address this, we impose the additional assumption that $P_\pi^n$ becomes strictly positive under sufficiently large $n$. Under this condition, we can characterize $\eta_0$.

\begin{assumption} \label{as: P strictly positive}
    $P_\pi^n$ is strictly positive.
\end{assumption}
Such an $n$ is guaranteed to exist by the irreducibility of $P_\pi$ from Assumption \ref{as: markov} \citep{levin2017markov}.
\begin{theorem}
    Let Assumption \ref{as: markov} and \ref{as: P strictly positive} hold. Then,  $A = B + \eta d_\mu e^\top$ is strictly positive stable for $\eta \in (0, \eta_0]$ where $\eta_0 \doteq 2 \min_{i,j}P_\pi^n(i,j)$.
\end{theorem}
\begin{proof}
    To prove that $A$ is strictly positive stable with the addition of Assumption \ref{as: P strictly positive}, we will once again utilize Lemma \ref{thm: bierkens}. Lemma~\ref{lem: first three conds} shows that $B$ defined in \eqref{eq: B definition}, $v = \eta d_\mu,$ and $w=e$ satisfy the first three conditions of Lemma~\ref{thm: bierkens}. In addition to Conditions 1-3, we will also prove that Condition 5 holds, which is sufficient to guarantee the strict positive stability of $A$. 
    Because $P_\pi^{n}$ is strictly positive, we have
\begin{align}
  K_{ij}
  &= \qty(1-d_\mu(i))\mathbb{I}\qty{i=j}
  + d_\mu(i)P_\pi^{\,n}(i,j) \\
  &\ge d_\mu(i)\,p_{\min}
  \quad \forall\, i,j .
\end{align}
where we define $p_{\min} \doteq \min_{i,j}P_\pi^n(i,j)>0$.
For any $\eta\in(0,\eta_0]$, we therefore have
\begin{align}
  2 K_{ij}\ge 2 d_\mu(i) p_{\min} \ge \eta d_\mu(i) =v_i w_j, \, \forall\, i,j,
\end{align}
so the entry-wise inequality in Condition 5
of Lemma~\ref{thm: bierkens} holds, and the theorem follows.
\end{proof}

Having established that $A$ is strictly positive–stable, the almost‐sure convergence of \eqref{eq: diff TD} to $v_\infty$ follows immediately by the same argument used in Corollary \ref{corr: on policy convergence}.
\begin{corollary} \label{corr: off policy conv eta}
Let Assumptions \ref{as: markov}, \ref{as: lr}, and \ref{as: P strictly positive} hold. Then for $\eta \in (0,\eta_0]$, where $\eta_0 \doteq 2 \min_{i,j}P_\pi^n(i,j)$,
the iterates $\qty{v_t}$ in \eqref{eq: diff TD} satisfy: $
        \lim_{t \rightarrow \infty} v_t = v_\infty ,\ a.s. $, where $v_\infty \in \fV_*$.
\end{corollary}

Having presented two sufficient conditions on $\eta$ and $n$,
we now present the third sufficient condition on $P_\pi$.
Namely, 
under the assumption that $P_\pi$ is doubly stochastic, we are able to establish the strict positive stability of $A$ for any $\eta \geq 0$ and $n > 0$.

\begin{assumption} \label{as: doubly stochastic}
    $P_\pi^n$ is doubly stochastic (i.e. $P_\pi^n e = e$ and $e^\top P_\pi^n = e^\top$).
\end{assumption}
Admittedly, doubly stochastic matrices are a small portion of the transition matrices considered in RL. They do arise, however, in simple random walks on $k$-regular, undirected graphs, such as cycles and complete graphs \citep{levin2017markov}. Furthermore,
doubly stochastic matrices are also a popular mathematical model (Section 8.7 \cite{horn2012matrix}).


\begin{theorem}
    Let Assumptions \ref{as: markov}, \ref{as: lr}, and \ref{as: doubly stochastic} hold. Then $A$ is strictly positive-stable for every $n$ and $\eta > 0$.
\end{theorem}

\begin{proof}
By the Lyapunov theorem (Theorem 4.6 in \citet{khalil2002nonlinear}), $A$ is positive‐stable if and only if there exists a symmetric positive‐definite matrix $M$ such that $
A^\top M + M A \succ 0$.
In the off‐policy case we take $M \doteq D_\mu^{-1}$.  Then with the definition of $A$ from \eqref{eq: A definition}, we have
\begin{align}
A^\top M + M A
&= (I - P_\pi^{n\top} + \eta\,e e^\top)
+ (I - P_\pi^n + \eta\,e e^\top).
\end{align}
We now show that $(I - P_\pi^n + \eta e e^\top)$ is positive‐definite.  When $P_\pi^n$ is doubly stochastic (so $\norm{P_\pi^n} = 1$), for any nonzero $v \in \R^\ns$,
\begin{align}
    v^\top \qty(I - P_\pi^n + \eta\,e e^\top)v
    &= \textstyle v^{\top}v - v^{\top}P_\pi^n v+ \eta(v^\top e)^2 \\
    &\ge \textstyle \norm{v}^{2} - \norm{P_\pi^n}\norm{v}^{2} + \eta (v^\top e)^2\\
          & =(1-\norm{P_\pi^n})\norm{v}^{2} + \eta (v^\top e)^2 \\
    &= \eta (v^\top e)^2> 0
\end{align}
where the first inequality holds by Cauchy-Schwarz, and the second equality holds because $\norm{P_\pi^n} = 1$.
\end{proof}

\begin{corollary}
    Let Assumptions \ref{as: markov}, \ref{as: lr}, and \ref{as: doubly stochastic} hold. Then the iterates $\qty{v_t}$ in \eqref{eq: diff TD} satisfy: $\lim_{t \rightarrow \infty} v_t = v_\infty ,\ a.s. $, where $v_\infty \in \fV_*$.
\end{corollary}

\section{Challenges and Open Problems} \label{sec: challenges}
We note that our off‐policy convergence guarantee in Corollary \ref{corr: off policy conv eta} rests on the conservative bound
$
\textstyle \eta \le \eta_{0}=2\min_{i,j}P_{\pi}^{\,n}(i,j),
$
which requires $P_\pi^n$ to be strictly positive (Assumption \ref{as: P strictly positive}). 
In this section, we demonstrate that empirically, this estimate for the upper bound of $\eta$ is quite pessimistic. 
We consider a $5\times5$ gridworld and set $n=3$.
Notably, an agent cannot reach every state in three steps.
So $\eta_0 = \min_{i,j}P_\pi^3(i,j)=0$ in this environment and Assumption \ref{as: P strictly positive} is violated.  
However, as Figure \ref{fig: eta} shows, the algorithm still converges for a wide range of $\eta$ values.

This empirical result seems to suggest that the convergence can be obtained for any $\eta$.
Furthermore,
\citet{wan2020learning} also prove the convergence for any $\eta$ with local-clock-based learning rates.
Then we might expect that some future work may be able to prove the convergence for any $\eta$ without the local clock as well.
However,
we must be cautious here.
\citet{anehila2022note} exhibit $M$-matrix counterexamples satisfying Conditions 1–3 for which $B + tv w^{\top}$ fails to remain strictly positive-stable once $t$ exceeds some finite threshold. This implies our Lemma \ref{lem: exists eta0 bound} \emph{may} not hold for large $\eta$.  Crucially, the linear algebra community still lacks a tight upper bound on admissible $\eta$ and has no known necessary and sufficient characterization of triples $(B,v,w)$ that ensure stability under rank-one perturbation \citep{anehila2022note}. Closing that theoretical gap would immediately yield sharper convergence guarantees here, and thus represents an important direction for future work.
\begin{figure}[htbp]
\centering
\includegraphics[width=0.75\linewidth]{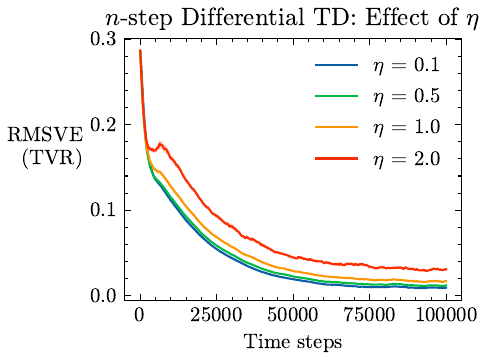}
\caption{Off‐policy convergence of $n$–step differential TD in a $5\times5$ gridworld ($n=3$) for various $\eta$. Although $\eta_0=0$ here, the algorithm is stable across $\eta$. We use a variant of root mean-squared value error from \citet{tsitsiklis1999average}, denoted as ‘RMSVE (TVR)’, which
measures the distance of the estimated values to the nearest solution that satisfies the Bellman equation
\eqref{eq bellman equation}. The trials are averaged over 30 seeds with shaded regions as 1 standard error. The experimental details and results for other $n$ values appear in Appendix C.}
\label{fig: eta}
\end{figure}
\section{Related Work} \label{sec: related}
\paragraph{Average Reward RL.}
Several temporal difference methods have been proposed for Markov decision processes with an average-reward objective. The best known is the average reward TD algorithm of \citet{tsitsiklis1999average}, whose convergence guarantees were first analyzed in the linear function approximation case, and further extended to the tabular setting by \citet{blaser2026asymptotic}. 
The differential TD algorithm we analyze here belongs to the same family but estimates the average reward with the full temporal-difference error instead of only using the reward sample \citep{wan2020learning}. Additional TD-based algorithms for policy evaluation and control in the average-reward setting include \citet{konda2000actor,abounadi2001learning,yang2016efficient,wan2021average,zhang2021policy,zhang2020average,zhang2021breaking,he2023loosely,saxena2023off,xie2025finite}.
\paragraph{Convergence of RL Algorithms.} The investigation of the almost sure convergence of RL algorithms is an active area of research.
Most prior work relies on the ODE based approach \citep{benveniste1990MP,kushner2003stochastic,borkar2009stochastic,liu2024ode},
where the corresponding ODE is relatively easy to analyze \citep{tsitsiklis1997analysis,konda2000actor,sutton2009convergent,sutton2009fast,zhang2020gradientdice,zhang2019provably,maei2011gradient,zhang2021truncated}.
By contrast,
the ODE studied in this work is highly nontrivial to analyze, and and we still do not have a complete characterization of it.
Other notable works involving nontrivial ODEs include \citet{meyn2024projected,wang2024almost}.
In addition to the ODE based approach,
the Robbins-Siegmund theorem \citep{robbins1971convergence} and its variant \citep{liu2025extensions} are gaining increasing attention for establishing almost sure convergence \citep{bertsekas1996neuro,qian2024almost,qian2025revisiting,liu2025extensions}, and has recently been formally verified \citep{zhang2025towards}. 
Beyond (asymptotic) almost sure convergence, the $L^2$ convergence rates of RL algorithms are also widely studied.
Notable works include \citet{mahadevan2014proximal, liu2015finite,wang2017finite,srikant2019finite,zou2019finite,wu2020finite,zhang2022globaloptimalityfinitesample,xie2025finite,liu2025linearq}.

\paragraph{Matrix Stability Under Perturbations and $D$-stability.}
The stability question in our paper lies within the broader $D$-stability problem which asks whether a given real matrix $A \in \mathbb{R}^{n \times n}$ remains strictly positive stable under left multiplication by any positive diagonal matrix $D \succ 0$. Despite Johnson’s necessary and sufficient criteria in low dimensions \citep{johnson1974sufficient,johnson1974d}, the general case ($n>4$) remains open, see \citet{hershkowitz1992recent} and \citet{kushel2019unifying} for comprehensive surveys. Our analysis is based on the results of \citet{bierkens2014singular} who investigated the $D$-stability of $M$ matrices under nonnegative rank-one perturbations. Their work extends a broader area of research investigating the eigenvalues and Jordan structure of rank-one perturbations of matrices \citep{moro2003low, savchenko2004change, ding2007eigenvalues, mehl2011eigenvalue, ran2012eigenvalues, fourie2013rank, mehl2014eigenvalue, ran2021global}.

\section{Conclusion}
Learning rates that use a local clock have played an essential role in the theoretical analysis of differential TD \citep{wan2020learning}, yet they remain largely unused by practitioners. 
This work bridges that divide by applying $D-$stability and rank-one perturbation theory from the linear algebra community, to provide novel almost sure convergence results of differential TD. To our knowledge, this is the first use of $D-$stability and rank-one perturbation techniques in RL. We expect this approach to enable further theoretical advances in RL, such as convergence proofs for differential Q-learning \citep{wan2020learning} and for RVI Q-learning \citep{abounadi2001learning} without relying on a local clock.

\subsubsection*{Acknowledgements}
EB acknowledges support from the NSF Graduate Research Fellowship under award 1842490. This work is supported in part by the US National Science Foundation under the awards III-2128019, SLES-2331904, and CAREER-2442098, the Commonwealth Cyber Initiative's Central Virginia Node under the award VV-1Q26-001, a Cisco Faculty Research Award, and an NVIDIA academic grant program award.

\bibliography{bibliography}





\clearpage
\appendix
\thispagestyle{empty}

\onecolumn
\aistatstitle{Appendix}

\section{Mathematical Background}

\subsection{Main Results from \citet{liu2024ode}}
First we will restate the main results from \citet{liu2024ode} concerning the convergence of SA iterates of the form \eqref{eq: SA} for completeness.
\begin{assumption} \label{as: liu markov}
    The Markov chain $\qty{Y_n}$ has a unique invariant probability measure (i.e. stationary distribution), denoted by $d_\mathcal{Y}$.
\end{assumption}

\begin{assumption} \label{as: liu lr}
    The learning rates $\qty{\alpha_n}$ are positive, decreasing and satisfy
    \begin{align}
        \sum_{i=0}^\infty \alpha_i = \infty, \lim_{n \rightarrow \infty} \alpha_n = 0, \, \text{and} \, \frac{\alpha_n - \alpha_{n+1}}{\alpha_n} = \fO(\alpha_n).
    \end{align}
\end{assumption}

\begin{assumption} \label{as: liu Hc}
Let $H_c(x,y) \doteq \frac{1}{c}H(cx,y)$. There exists a measurable function $H_\infty(x,y) \doteq \lim_{c\rightarrow \infty} H_c(x,y)$, a scalar function $\kappa\colon \R\to\R$ (independent of $x,y$), and a measurable function $b(x,y)$ such that for any $x,y$:
\begin{align}
H_c(x,y) - H_\infty(x,y) &= \kappa(c)b(x,y), \label{eq:ass3a}\\
\lim_{c\to\infty}\kappa(c) &= 0. \label{eq:ass3b}
\end{align}
Moreover, there exists a measurable function $L_b(y)$ such that for all $x,x',y$,
\begin{align}
\bigl\lVert b(x,y) - b(x',y)\bigr\rVert
&\le L_b(y)\,\lVert x - x'\rVert, \label{eq: as lipschitz 3}
\end{align}
and the expectation,
\begin{align}
L_b \doteq \E_{y\sim d_y}\bigl[L_b(y)\bigr]
\end{align}
is well-defined and finite.
\end{assumption}

\begin{assumption} \label{as: liu lipschitz}
There exists a measurable function $L(y)$ such that for any $x,x',y$,
\begin{align}
\bigl\lVert H(x,y) - H(x',y)\bigr\rVert &\le L(y)\,\lVert x - x'\rVert,\label{eq:ass4a}\\
\bigl\lVert H_\infty(x,y) - H_\infty(x',y)\bigr\rVert &\le L(y)\,\lVert x - x'\rVert.\label{eq:ass4b}
\end{align}
Moreover, the following expectations are well‐defined and finite for every $x$:
\begin{align}
h(x)&\doteq\E_{y\sim d_y}\bigl[H(x,y)\bigr], \label{eq: E H}\\
h_\infty(x)&\doteq\E_{y\sim d_y}\bigl[H_\infty(x,y)\bigr], \label{eq: E H infty}\\
L&\doteq\E_{y\sim d_y}\bigl[L(y)\bigr]. \label{eq: E L}
\end{align}
\end{assumption}

\begin{assumption} \label{as: liu ode conv}
    As $c \rightarrow \infty$, $h_c(x)$ converges to $h_\infty(x)$ uniformly on $x$ on any compact subsets of $\R^d$. The ODE,
    \begin{align}
        \dv{x(t)}{t} = h_\infty(x(t))
    \end{align}
    has 0 as its G.A.S equilibrium.
\end{assumption}

\begin{assumption} \label{as: liu lil}

Let $g$ denote any of the following functions:
\begin{align}
y &\mapsto H(x,y)\quad(\forall x), \label{eq:ass6a}\\
y &\mapsto L_b(y), \label{eq:ass6b}\\
y &\mapsto L(y). \label{eq:ass6c}
\end{align}
Then for any initial condition $Y_1$, it holds that
\begin{align}
\lim_{n\to\infty}\alpha_n\sum_{i=1}^n\Bigl(g(Y_i)-\E_{y\sim d_y}[g(y)]\Bigr)
=0\quad\text{a.s.} \label{eq:ass6d}
\end{align}

\end{assumption}

\begin{lemma} \label{lem: liu ode convergence}
    Let Assumptions \ref{as: liu markov} - \ref{as: liu lil} hold. Then the iterates $\qty{x_n}$ generated by \eqref{eq: SA} converge almost surely to a (sample-path-dependent) bounded invariant set of the ODE
    \begin{align}
        \dv{x(t)}{t}= h(x(t)).
    \end{align}
\end{lemma}

\section{Derivation of n-step Differential TD} \label{ap: derivation}
We begin with the Bellman equation for differential TD in matrix-vector form \eqref{eq bellman equation}.
Recall that $P_\pi \in [0,1]^{\ns\times\ns}$ denotes the stochastic matrix of the Markov chain induced by the target policy $\pi$, $r_\pi \in \R^{\ns}$ represents the expected rewards under $\pi$, and $\bar{J}_\pi$ is the average reward.
The one-step Bellman equation of $v_\pi$ is
\begin{align}
    v_\pi = r_\pi - \bar{J}\pi e + P_\pi v_\pi.
\end{align}
We can keep unrolling it for $n$ steps and get
\begin{align}
     v_\pi &= r_\pi - \bar{J}_\pi e + P_\pi \qty(r_\pi - \bar{J}_\pi e  + P_\pi v_\pi)\\
      &= r_\pi - 2 \bar{J}_\pi e + P_\pi r_\pi + P_\pi^2 v_\pi \\
      &\quad \vdots\\
      &= -n \bar{J}_\pi e + r_\pi + P_\pi r_\pi + P_\pi^2 r_\pi + \cdots + P_\pi^{n-1} r_\pi + P_\pi^n v_\pi \label{eq: unroll v for n steps}
\end{align}
Hence, by~\eqref{eq: unroll v for n steps}, we have for all $s \in \fS$, 
\begin{align}
    v_\pi(s) =& \E\qty[\sum_{i = 0}^{n-1}\qty(r(S_{i}, A_{i}) - \bar{J}_\pi) + v_\pi(S_{n})\Bigg| 
    S_0 = s, A_{i} \sim \pi(\cdot \mid S_{i}), S_{i+1} \sim p(\cdot \mid S_i, A_i)]\\
    =& v_\pi(s) + \E\qty[\sum_{i = 0}^{n-1}\qty(r(S_{i}, A_{i}) - \bar{J}_\pi) + v_\pi(S_n) - v_\pi(S_0)\Bigg| 
    S_0 = s, A_i \sim \pi(\cdot \mid S_{i}), S_{i+1} \sim p(\cdot \mid S_{i}, A_{i})]\\
    =& v_\pi(s) + \E\qty[\rho_{0:n-1}\qty(\sum_{i = 0}^{n-1}\qty(r(S_{i}, A_{i}) - \bar{J}_\pi) + v_\pi(S_{n}) - v_\pi(S_0))\Bigg| 
    S_0 = s, A_{i} \sim \mu(\cdot \mid S_{i}), S_{i+1} \sim p(\cdot \mid S_{i}, A_{i})].
\end{align}

Therefore, we have the $n$-step bootstrapped differential TD update
\begin{align}
    v_{t+n}(S_t) = v_{t+n-1}(S_t) + \alpha_{t+n-1}\rho_{t:t+n-1}\qty(R_{t+1:t+n} - n J_{t+n-1} + v_{t+n-1}(S_{t+n}) - v_{t+n-1}(S_t))
\end{align}
where $J_t$ is the average reward estimate to be defined shortly.

Let $d_\mu \in [0,1]^\ns$ denote the stationary distribution induced by the behavior policy $\mu$.
Rearranging~\eqref{eq: unroll v for n steps}, we get
\begin{align}
    n \bar{J}_\pi e &= r_\pi + P_\pi r_\pi + P_\pi^2 r_\pi + \cdots + P_\pi^{n-1} r_\pi + P_\pi^n v_\pi - v_\pi \\
    \bar{J}_\pi e &= \frac{1}{n} \left( r_\pi + P_\pi r_\pi + P_\pi^2 r_\pi + \cdots + P_\pi^{n-1} r_\pi + P_\pi^n v_\pi - v_\pi \right) \\
    d_\mu^\top \bar{J}_\pi e &= \frac{d_\mu^\top}{n} \left( r_\pi + P_\pi r_\pi + \cdots + P_\pi^{n-1} r_\pi + P_\pi^n v_\pi - v_\pi \right) \\
    \bar{J}_\pi &= \frac{d_\mu^\top}{n} \left( r_\pi + P_\pi r_\pi + \cdots + P_\pi^{n-1} r_\pi + P_\pi^n v_\pi - v_\pi \right)\label{eq: J definition}
\end{align}
Therefore, by~\eqref{eq: J definition}, we have
\begin{align}
    \bar{J}_\pi 
    =& \frac{1}{n}\E\qty[\sum_{i = 0}^{n-1}r(S_{i}, A_{i}) + v_\pi(S_{n}) -v_\pi(S_0)\Bigg| 
    S_0 \sim d_\pi, A_{i} \sim \pi(\cdot \mid S_{i}), S_{i+1} \sim p(\cdot \mid S_{i}, A_{i})]\\
    =& \bar{J}_\pi + \frac{1}{n}\E\qty[\sum_{i = 0}^{n-1}\qty(r(S_{i}, A_i) - \bar{J}_\pi) + v_\pi(S_{n}) - v_\pi(S_0)\Bigg| 
    S_0 \sim d_\pi, A_{i} \sim \pi(\cdot \mid S_{i}), S_{i+1} \sim p(\cdot \mid S_{i}, A_{i})]\\
    =& \bar{J}_\pi + \frac{1}{n}\E\qty[\rho_{0:n-1}\qty(\sum_{i = 0}^{n-1}\qty(r(S_{i}, A_{i}) - \bar{J}_\pi) + v_\pi(S_{n}) - v_\pi(S_0))\Bigg| 
    S_0 \sim d_\mu, A_{i} \sim \mu(\cdot \mid S_{i}), S_{i+1} \sim p(\cdot \mid S_{i}, A_{i})].
\end{align}

As a result, we update the average reward estimate as
\begin{align}
    J_{t+n}= J_{t+n-1} + \frac{\eta}{n}\alpha_{t+n-1}\rho_{t:t+n-1}\qty(R_{t+1:t+n} - n J_{t+n-1} + v_{t+n-1}(S_{t+n}) - v_{t+n-1}(S_t)),
\end{align}
where $\eta$ is a positive multiplicative constant to allow for a different update rate relative to $v$.


\section{Technical Lemmas}

\begin{lemma}\label{lem:Hc}
    The function $H(v,y)$ defined in~\eqref{eq: H def} satisfies Assumption~\ref{as: liu Hc}.
\end{lemma}
\begin{proof}
    Recall that for a scalar $c > 0$, the scaled operator is defined as $H_c(x,y) \doteq \frac{1}{c} H(cx, y)$. Substituting the definition of $H$ from \eqref{eq: H def}, we obtain
\begin{align}
H_c(v,y)[s] = \rho_{0:n-1}(y)\qty( \frac{1}{c} \sum_{k=0}^{n-1} r(s_k, a_k) - \eta \sum v + v(s_n) - v(s_0))\mathbb{I}\qty{s = s_0}.
\end{align}
This implies that the limit of the operator $H_\infty(v, y)= \lim_{c \rightarrow \infty} H_c(v,y)$ is given by
\begin{align}
H_\infty(v,y)[s] \doteq \rho_{0:n-1}(y)\qty( - \eta \sum v + v(s_n) - v(s_0) )\mathbb{I}\{s = s_0\}, \label{eq: H infty}
\end{align}
and the difference $H_c(v,y) - H_\infty(v,y)$ simplifies to
\begin{align}
H_c(v,y)[s] - H_\infty(v,y)[s] = \rho_{0:n-1}(y)\qty( \frac{1}{c} \sum_{k=0}^{n-1} r(s_k, a_k) )\mathbb{I}\{s = s_0\}.
\end{align}
Therefore, Assumption~\ref{as: liu Hc} \eqref{eq:ass3b} is satisfied by setting $\kappa(c) \doteq \frac{1}{c}$, which vanishes as $c \to \infty$, and defining
\begin{align}
b(x,y)[s] \doteq \rho_{0:n-1}(y) \qty( \sum_{k=0}^{n-1} r(s_k, a_k) ) \mathbb{I}\{s = s_0\}.
\end{align}
Importantly, $b(x,y)$ is independent of $x$, so for all $x,x',y$ we have $\|b(x,y) - b(x',y)\| = 0$. Thus, the Lipschitz condition in \eqref{eq: as lipschitz 3} from Assumption~\ref{as: liu Hc} is trivially satisfied with $L_b(y) = 0$, and the expectation $\E[L_b(y)]$ is finite.

\end{proof}

\begin{lemma}\label{lem:H-lipschitz}
    The function $H(v,y)$ defined in~\eqref{eq: H def} satisfies Assumption~\ref{as: liu lipschitz}.
\end{lemma}
\begin{proof}
We first verify that $H$ is Lipschitz continuous in the $\infty-$norm, i.e. \eqref{eq:ass4a}. 
Fix any transition $y = (s_0,a_0,s_1, \dots s_n)$ and any state $s$. Using the indicator $\qty (\mathbb{I}\qty{s=s_0})$, we have
\begin{align}
H(v)[s] - H(w)[s]
&= \rho(y)\Bigl[\qty(\sum_{k=0}^{n-1} r\qty(s_k, a_k)-\eta \sum v + v(s_n)-v(s_0)) \\
&\quad-\qty(\sum_{k=0}^{n-1} r\qty(s_k, a_k)-\eta \sum w + w(s_n)-w(s_0))\Bigr]\mathbb{I}\qty{s=s_0}\\
&=
\begin{cases}
\rho(y)\bigl[-\eta \sum(v-w) + (v(s_n)-w(s_n)) - (v(s_0)-w(s_0))\bigr] , & s = s_0,\\
0, & s \neq s_0.
\end{cases}
\end{align}
Hence
\begin{align}
\bigl|H(v)[s]-H(w)[s]\bigr|
&\le
\begin{cases}
\rho(y)\bigl[\eta \ns \norm{v-w}_\infty + 2\,\norm{v-w}_\infty\bigr], & s = s_0,\\
\norm{v-w}_\infty, & s \neq s_0.
\end{cases}
\end{align}

The only remaining task is to upper-bound $\rho(y)$ defined in \eqref{eq: H def}.
Under our standard “coverage’’ assumption (Assumption \ref{as: markov}), whenever $\pi(a|s) >0$, we also have $\mu(a|s)>0$, and because both $\ns$ and $\na$ are finite, there is a uniform lower bound
\begin{align}
    \mu_{\min} &= \min_{s,a:\, \pi(a\mid s)>0}\mu(a\mid s)>0.
\end{align}
This implies that
\begin{align}
\rho(y)
=\prod_{k=0}^{n-1}\frac{\pi(a_k\mid s_k)}{\mu(a_k\mid s_k)}
\le \qty(\frac{1}{\mu_{\min}})^{n}
\doteq \bar\rho<\infty.
\end{align}
Therefore,
    \begin{align}
        \norm{H(v,y)-H(w,y)}_\infty &= \max_{s \in \fS} \abs{H(v,y)[s] - H(w,y)[s]} \\ &\leq L \norm{v-w}_\infty,
    \end{align}
    where $L \doteq \max\qty{1,\bar\rho\qty(\eta \ns +2)}.$

Verifying the Lipschitz continuity of $H_\infty$, starting from \eqref{eq: H infty} from the proof of Lemma \ref{lem:Hc} to avoid repetition, we have,
\begin{align}
    H_{\infty}(v)[s] - H_{\infty}(w)[s]
    &= \rho(y)\Bigl[\qty(-\eta \sum v + v(s_n)-v(s_0))-\qty(-\eta \sum w + w(s_n)-w(s_0))\Bigr]\mathbb{I}\qty{s=s_0}\\
    &=
    \begin{cases}
    \rho(y)\bigl[-\eta \sum(v-w) + (v(s_n)-w(s_n)) - (v(s_0)-w(s_0))\bigr] , & s = s_0,\\
    0, & s \neq s_0. 
    \end{cases} \\
    &= H(v)[s] - H(w)[s]
\end{align}
Therefore, the Lipschitz continuity of $H_\infty$ holds for $L$ as well, proving \eqref{eq:ass4b} holds. 

From equation \eqref{eq: h}, its easy to see that $h$ is finite, verifying \eqref{eq: E H}. To verify \eqref{eq: E H infty}, its also easy to see that 
\begin{align}
\E_{y\sim d_{\fY}}\!\bigl[ H_\infty(v,y)[s]\bigr]
&=\E_{\mu}\!\Bigl[
      \rho_{0:n-1}(y)\,
      \bigl(-\eta\!\textstyle\sum_{p}v(p)+v(s_n)-v(s_0)\bigr)
      \mathbf \mathbb{I}\{s=s_0\}
   \Bigr]\\
&=d_\mu(s)\,
   \E_{\pi}\!\Bigl[
      -\eta\,e^{\!\top}v
      +v(s_n)-v(s)
      \;\Bigm|\;s_0=s
   \Bigr]\\
&=d_\mu(s)\Bigl[
      -\eta\,e^{\!\top}v
      +(P_\pi^{\,n}v)(s)-v(s)
   \Bigr]\\
&=\bigl[D_\mu\bigl(P_\pi^{\,n}-I-\eta\,ee^{\!\top}\bigr)v\bigr](s).
\end{align}
In vector form,
\begin{align}
\E_{y\sim d_{\fY}}\bigl[ H_\infty(v,y)\bigr]
=D_\mu\bigl(P_\pi^{\,n}-I-\eta\,ee^{\!\top}\bigr)v,
\end{align}
which is clearly finite. Finally, since the Lipschitz constant $L$ is independent of $y$, \eqref{eq: E L} holds trivially.
\end{proof}

\section{Experimental Details and Additional Results} \label{ap: exp details}
\subsection{Experimental Details} \label{sec: exp description}
 Our experiments were carried out in a simple, continuing gridworld. At each timestep $t$ the agent occupies one cell in an $5\times 5$ grid, can move to any of the four orthogonal neighbors (subject to walls at the borders where an illegal move keeps the agent in the same state), and receives a unit reward each time it reaches the designated “goal” state in the bottom right corner.  Rather than terminating, after reaching the goal state, the agent is put back to the start corner (top left) on the very next step with probability 1, making this a continuing task.

We evaluated an off-policy, multi-step differential TD learner for a fixed $\epsilon$-greedy target policy with $\epsilon = 0.1$.  The behavior policy is random. The value function and reward estimate were both initialized to zero. Throughout training, the agent took actions following the random behavior policy, observed the deterministic next state and reward, and performed its $n$-step TD updates with Off-Policy Differential TD to estimate the value function associated with the target policy. The experiment was run for $100,000$ steps, and repeated with 30 random seeds. For Figure \ref{fig: eta}, the $\eta$ values were reported for $\eta = \qty{0.1, 0.5, 1, 2}$, while keeping $n=3$. In Figure \ref{fig: n}, $\eta = 0.1$ and $n = \qty{1,2,3,4}$. In both experiments, we used a constant learning rate $\alpha = 0.01$. 

We now elaborate on the evaluation metric. The convergence was assessed using a variant of root-mean-squared value error from \citep{tsitsiklis1999average} which we denote as `RMSVE (TVR)', which is also used in \citep{wan2020learning}. As noted in Section \ref{sec: background}, the solutions to the differential Bellman equation \eqref{eq bellman equation} form a set $\fV_* = \qty{v_\pi + ce}$. Which point in this set an algorithm converges to depends on initializations and the design choices of the algorithm. Therefore, computing the value error with respect to $v_\pi$ does not say much about convergence. To remedy this, \citep{tsitsiklis1999average} proposed computing the error
with respect to the nearest valid solution to the Bellman equations. The metric is defined as
\begin{align}
    \text{RMSVE(TVR)} (v, v_\pi) \doteq \inf_{c}\norm{v-(v_\pi + ce)}_{d_\pi}.
\end{align}
Algorithmically, this amounts to computing the offset of the learned value function, subtracting it, and then computing the RMSVE with respect to $v_\pi$.
As demonstrated in Section C.4 of \citet{wan2020learning} this $v_\pi$ can be analytically computed using the Bellman equations with the additional constraint that $d_\pi^\top v_\pi = 0$ (effectively centering the value function).

\subsection{Additional Experiment on the effect of $n$}
Here, we present additional experiments, demonstrating that the conclusions from Section \ref{sec: challenges} also hold for various $n$ values. Despite the fact that $\eta_0 = 0$, we empirically observe the convergence of differential TD with $\eta=0.1$ under several choices of $n$. In this experiment, we used the environment described in Section \ref{sec: exp description}.
\begin{figure}[htbp]

\centering
\includegraphics[width=0.4\linewidth]{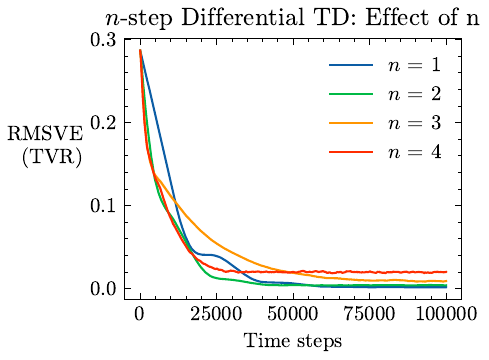}
\caption{Off‐policy convergence of $n$–step differential TD in a $5\times5$ gridworld for various $n$ with fixed $\eta = 0.1$. See Section \ref{sec: exp description} for the complete experiment description. Despite $\eta_0 =0$, we still observe the convergence of differential TD with across several $n$ values.
}
\label{fig: n}
\end{figure}

\end{document}